\documentclass[twoside]{article}
\usepackage[accepted]{aistats2012}
\usepackage[colorlinks,citecolor=blue,urlcolor=red]{hyperref}
\usepackage{cleveref}
\usepackage{mynotation}

\usepackage[
  style=numeric,
  maxbibnames=99,
  firstinits=true,
  doi=false, isbn=false, url=false, eprint=false
]{biblatex}

\defbibheading{bibliography}{%
  \subsubsection*{References}%
}

\begin{document}

\twocolumn[

\aistatstitle{Minimax Rates of Estimation for Sparse PCA in High Dimensions}
\aistatsauthor{ Vincent Q. Vu \And Jing Lei }
\aistatsaddress{ Department of Statistics\\Carnegie Mellon University\\{vqv@stat.cmu.edu} \And
                Department of Statistics\\Carnegie Mellon University\\{jinglei@andrew.cmu.edu} }
]

\runningtitle{Minimax Rates for Sparse PCA}

\begin{abstract}
We study sparse principal components analysis in the high-dimensional
setting, where $p$ (the number of variables) can be much larger 
than $n$ (the number of observations).
We prove optimal, non-asymptotic lower and upper bounds on the minimax estimation error for
the leading eigenvector when it belongs to an $\ell_q$ ball for $q \in
[0,1]$. Our bounds are sharp in $p$ and $n$ for all $q \in [0, 1]$
over a wide class of distributions. The upper bound is obtained by analyzing 
the performance of $\ell_q$-constrained PCA. In particular, our results
provide convergence rates for $\ell_1$-constrained
PCA.

\end{abstract}

\section{Introduction}
\label{sec:introduction}
High-dimensional data problems, where the number of variables $p$
exceeds the number of observations $n$, are pervasive in
modern applications of statistical inference and machine learning.
Such problems have increased the necessity of dimensionality reduction for
both statistical and computational reasons. In some applications,
dimensionality reduction is the end goal, while in others it is
just an intermediate step in the analysis stream. In either case, 
dimensionality reduction is usually data-dependent and so
the limited sample size and noise may have an adverse affect.
Principal components analysis (PCA) is perhaps one of the most well known and
widely used techniques for unsupervised dimensionality reduction.  
However, in the high-dimensional situation, where $p/n$ does not tend to 0 as
$n \to \infty$, PCA may not give consistent estimates of eigenvalues and
eigenvectors of the population covariance matrix \cite{Johnstone:2009}. To
remedy this situation, sparsity constraints on estimates of the leading
eigenvectors have been proposed and shown to perform well in various
applications. In this paper we prove optimal minimax error bounds for sparse PCA
when the leading eigenvector is sparse.

\subsection{Subspace Estimation}
Suppose we observe i.i.d. random vectors $X_i \in \Real^p$, $i = 1,\ldots, n$
and we wish to reduce the dimension of the data from $p$ down to $k$.
PCA looks for $k$ uncorrelated, linear combinations of the $p$
variables that have maximal variance.
This is equivalent 
to finding a $k$-dimensional linear subspace whose orthogonal projection $A$
minimizes the mean squared error
\begin{equation}
  \label{eq:projection-error}
  \text{mse}(A) = \E \norm{(X_{i} - \E X_{i}) - A (X_{i} - \E X_{i})}_2^2
\end{equation}
\parencite[see][Chapter 7.2.3 for example]{Izenman:2008}.
The optimal subspace is determined by spectral decomposition of the population
covariance matrix
\begin{equation}
  \label{eq:pop-spectraldecomposition}
  \Sigma
  = \E X_i X_i^T - (\E X_i)(\E X_i)^T
  = \sum_{j=1}^p \lambda_j \theta_j \theta_j^T
  \,,
\end{equation}
where $\lambda_1 \geq \lambda_2 \geq \cdots \lambda_p \geq 0$ are the eigenvalues
and $\theta_1, \ldots \theta_p \in \Real^p$, orthonormal, are
eigenvectors of $\Sigma$. If $\lambda_k > \lambda_{k+1}$, then the optimal
$k$-dimensional linear subspace is the span of $\Theta = (\theta_1,\ldots,\theta_k)$ and its
projection is given by $\Pi = \Theta \Theta^T$.
Thus, if we know $\Sigma$ then we may optimally
(in the sense of \cref{eq:projection-error})
reduce the dimension of the data from $p$ to $k$ by the mapping $x \mapsto
\Theta \Theta^T x$.

In practice, $\Sigma$ is not known and so $\Theta$ must be estimated from the
data.  In that case we replace $\Theta$ by an estimate $\hat{\Theta}$ and
reduce the dimension of the data by the mapping $x \mapsto \hat{\Pi} x$, 
where $\hat{\Pi} = \hat{\Theta} \hat{\Theta}^T$.
PCA uses the spectral decomposition of the sample covariance matrix
\begin{equation*}
  S
  = \frac{1}{n} \sum_{i=1}^n X_i X_i^T - \bar{X} \bar{X}^T
  = \sum_{j=1}^{p\land n} l_j u_j u_j^T
  \,,
\end{equation*}
where $\bar{X}$ is the sample mean, and $l_j$ and $u_j$ are eigenvalues and eigenvectors
of $S$ defined analogously to \cref{eq:pop-spectraldecomposition}.
It reduces the dimension of the data to $k$ by the mapping $x \mapsto U U^T x$,
where $U = (u_1,\ldots,u_k)$.  

In the classical regime where $p$ is fixed and $n \to \infty$, PCA 
is a consistent estimator of the population eigenvectors.  However, 
this scaling is not appropriate for modern applications where $p$ is
 comparable to or larger than $n$. In that case, it has been observed
\cite{Paul:2007,Nadler:2008,Johnstone:2009} that if $p, n \to \infty$ and $p/n \to c > 0$,
then PCA can be an inconsistent estimator in the sense that the angle between
$u_1$ and $\theta_1$ can remain bounded away from $0$ even as $n \to
\infty$.

\subsection{Sparsity Constraints}
Estimation in high-dimensions may be beyond hope without additional structural
constraints.  In addition to making estimation feasible, these structural constraints may 
also enhance interpretability of the estimators.  One important example of this is 
sparsity. The notion of sparsity is that a few variables have large effects,
while most others are negligible. This type of assumption is often reasonable
in applications and is now widespread in high-dimensional
statistical inference.

Many researchers have proposed sparsity constrained versions of PCA along with
practical algorithms, and research in this direction continues to be very active
\cite[e.g.,][]{Jolliffe:2003,Zou:2006,dAspremont:2007,Shen:2008,Witten:2009}.
Some of these works are based on the idea of adding an $\ell_1$ constraint to
the estimation scheme. For instance, \textcite{Jolliffe:2003} proposed adding an
$\ell_1$ constraint to the variance maximization formulation of PCA. Others
have proposed convex relaxations of the ``hard'' $\ell_0$-constrained form of
PCA \cite{dAspremont:2007}. Nearly all of these proposals are based on an
iterative approach where the eigenvectors are estimated in a one-at-a-time
fashion with some sort of deflation step in between \cite{Mackey:2008}. 
For this reason, we consider the basic problem of estimating the leading 
population eigenvector $\theta_1$.

The $\ell_q$ balls for $q \in [0,1]$ provide an appealing way to make the
notion of sparsity concrete.  These sets are defined by
\begin{equation*}
  \Ball_q^p(R_q)
  = \{ \theta \in \Real^p : \textstyle{\sum_{j=1}^p} |\theta_j|^q \leq R_q \}
\end{equation*}
and
\begin{equation*}
  \Ball_0^p(R_0)
  = \{ \theta \in \Real^p : \textstyle{\sum_{j=1}^p} 1_{\{ \theta_j \neq 0\}} \leq R_0 \}
  \,.
\end{equation*}
The case $q = 0$ corresponds to ``hard'' sparsity where $R_0$ is the number of 
nonzero entries of the vectors.  For $q > 0$ the $\ell_q$ balls capture 
``soft'' sparsity where a few of the entries of $\theta$ are large, while most 
are small. The soft sparsity case may be more realistic for applications where the
effects of many variables may be very small, but still nonzero.

\subsection{Minimax Framework and High-Dimensional Scaling}
In this paper, we use the statistical minimax framework to elucidate the
difficulty/feasibility of estimation when the leading eigenvector $\theta_1$
is assumed to belong to $\Ball_q^p(R_q)$ for $q \in [0,1]$. The framework can
make clear the fundamental limitations of statistical inference that
\emph{any estimator} $\hat{\theta}_1$ must satisfy. Thus, it can reveal gaps
between optimal estimators and computationally tractable ones, and also
indicate when practical algorithms achieve the fundamental limits.

\paragraph{Parameter space}
There are two main ingredients in the minimax framework. The first is the class
of probability distributions under consideration. These are usually associated 
with some parameter space corresponding to the structural constraints.
Formally, suppose that $\lambda_1 > \lambda_2$. Then we may write
\cref{eq:pop-spectraldecomposition} as
\begin{equation}
\label{eq:model}
  \Sigma
  = \lambda_1 \theta_1\theta_1^T + \lambda_2 \Sigma_0\,,
\end{equation}
where $\lambda_1 > \lambda_2 \geq 0$, $\theta_1 \in \Sphere_2^{p-1}$ (the unit sphere of $\ell_2$),
$\Sigma_0 \succeq 0$, $\Sigma_0 \theta = 0$, and $\norm{\Sigma_0}_2 = 1$ 
(the spectral norm of $\Sigma_0$).
In model~\eqref{eq:model}, the covariance matrix $\Sigma$ has a unique largest
eigenvalue $\lambda_1$. 
Throughout this paper, for $q \in [0,1]$, we consider the class
\newcommand{\modelclass}[1]{\mathcal{M}_{#1}}
\begin{equation*}
  \modelclass{q}(\lambda_1, \lambda_2,\bar{R}_q, \alpha, \kappa)
\end{equation*}
that consists of all probability distributions on
$X_i \in \Real^p$, $i=1,\ldots,n$ satisfying model~\eqref{eq:model} with 
$\theta_1\in \Ball_q^p(\bar{R}_q+1)$, and
Assumption~\autoref{assumption:lq-radius} (below) with $\alpha$ and $\kappa$ depending on
$q$ only.

\paragraph{Loss function}
The second ingredient in the minimax framework is the loss function.
In the case of subspace estimation, an obvious criterion for evaluating the
quality of an estimator $\hat{\Theta}$ is the squared distance between
$\hat{\Theta}$ and $\Theta$. However, it is not appropriate because
$\Theta$ is not unique---$\Theta$ and $\Theta V$ span the same subspace 
for any $k \times k$ orthogonal matrix $V$.
On the other hand, the orthogonal projections $\Pi = \Theta\Theta^T$
and $\hat{\Pi} = \hat{\Theta} \hat{\Theta}^T$ are unique. So we consider the
loss function defined by the Frobenius norm of their difference:
\begin{equation*}
  \norm{\hat{\Pi} - \Pi}_F
  \,.
\end{equation*}
In the case where $k = 1$, the only possible non-uniqueness in the leading
eigenvector is its sign ambiguity. Still, we prefer to use the above loss
function in the form
\begin{equation*}
  \norm{\hat{\theta}_1 \hat{\theta}_1^T - \theta_1 \theta_1}_F
\end{equation*}
because it generalizes to the case $k > 1$.  Moreover, 
when $k = 1$, it turns out to be equivalent to both the Euclidean distance
between $\theta_1$, $\hat{\theta}_1$ (when they belong to the same half-space) 
and the magnitude of the sine of the angle between $\theta_1$, $\hat{\theta}_1$.
(See \Cref{lem:canonical_angles,lem:L2-equivalence} in the Appendix.)

\paragraph{Scaling}
Our goal in this work is to provide non-asymptotic bounds on the minimax error
\begin{equation*}
  \min_{\hat{\theta}_1} 
  \max_{P \in \modelclass{q}(\lambda_1, \lambda_2,\bar{R}_q, \alpha, \kappa)}
  \E_P \norm{\hat{\theta}_1 \hat{\theta}_1^T - \theta_1 \theta_1}_F
  \,,
\end{equation*}
where the minimum is taken over all estimators that depend only on $X_1,\ldots,X_n$,
that explicitly track the dependence of the minimax error on the vector 
$(p, n, \lambda_1, \lambda_2, \bar{R}_q)$.  As we stated early, 
the classical $p$ fixed, $n \to \infty$ scaling completely misses the effect
of high-dimensionality; we, on the other hand, want to highlight the role that
sparsity constraints play in high-dimensional estimation.
Our lower bounds on the minimax error use an information theoretic technique
based on Fano's Inequality. The upper bounds are obtained by constructing an
$\ell_q$-constrained estimator that nearly achieves the lower bound.

\subsection{$\ell_q$-Constrained Eigenvector Estimation}
Consider the constrained maximization problem
\begin{equation}
  \label{eq:lqpca}
  \begin{aligned}
    & \text{maximize}
    & & b^T S b, \\
    & \text{subject to}
    & & b \in \Sphere_2^{p-1} \cap \Ball_q^p(\rho_q)
  \end{aligned}
\end{equation}
and the estimator defined to be the solution of the optimization problem.
The feasible set is non-empty when $\rho_q \geq 1$, and the 
$\ell_q$ constraint is active only when $\rho_q \leq p^{1 - \frac{q}{2}}$.
The $\ell_q$-constrained estimator corresponds to ordinary PCA when $q = 2$
and $\rho_q = 1$. When $q \in [0,1]$, the $\ell_q$ constraint promotes sparsity
in the estimate. Since the criterion is a convex function of $b$, the 
convexity of the constraint set is inconsequential---it may be replaced by its 
convex hull without changing the optimum.

The case $q = 1$ is the most interesting from a practical point of view, 
because it corresponds to the well-known Lasso estimator for linear
regression. In this case, \cref{eq:lqpca} coincides with the method proposed by
\textcite{Jolliffe:2003}, though \eqref{eq:lqpca} remains a difficult \emph{convex 
maximization} problem.  Subsequent authors \cite{Shen:2008,Witten:2009} have 
proposed efficient algorithms that can approximately solve \cref{eq:lqpca}. Our 
results below are (to our knowledge) the first convergence rate results 
available for this $\ell_1$-constrained PCA estimator.

\subsection{Related Work}
\Textcite{Amini:2009} analyzed the performance of a semidefinite programming
(SDP) formulation of sparse PCA for a generalized spiked covariance model
\cite{Johnstone:2001}. Their model assumes that the nonzero entries of
the eigenvector all have the same magnitude, and that the covariance matrix
corresponding to the nonzero entries is of the form $\beta \theta_1\theta_1^T
+ I$. They derived upper and lower bounds on the success probability for model
selection under the constraint that $\theta_1 \in \Ball_0^p(R_0)$.  Their upper bound 
is conditional is conditional on the SDP based estimate being rank 1 . Model selection 
accuracy and estimation accuracy are different notions of accuracy. One does 
not imply the other. In comparison, our results below apply to a wider class of 
covariance matrices and in the case of $\ell_0$ we provide sharp bounds for the 
estimation error.

Operator norm consistent estimates of the covariance matrix automatically
imply consistent estimates of eigenspaces. This follows from matrix 
perturbation theory \cite[see, e.g.,][]{StewartAndSun}.
There has been much work on finding operator norm consistent covariance 
estimators in high-dimensions under assumptions on the sparsity or bandability of the entries of 
$\Sigma$ or $\Sigma^{-1}$ \cite[see, 
e.g.,][]{Bickel:2008a,Bickel:2008,ElKaroui:2008}. Minimax 
results have been established in that setting by \textcite{Cai:2010}. However, 
sparsity in the covariance matrix and sparsity in the leading eigenvector are 
different conditions. There is some overlap (e.g. the spiked covariance model), 
but in general, one does not imply the other.

\Textcite{Raskutti:2011} studied the related problem of minimax estimation for 
linear regression over $\ell_q$ balls. Remarkably, the rates that we derive for 
PCA are nearly identical to those for the Gaussian sequence model and 
regression. The work of \textcite{Raskutti:2011} is close to ours in that they 
inspired us to use some similar techniques for the upper bounds.

While writing this paper we became aware of an unpublished manuscript by 
\textcite{Paul:draft}.  They also study PCA under $\ell_q$ constraints with a 
slightly different but equivalent loss function.  Their work provides asymptotic 
lower bounds for the minimax rate of convergence over $\ell_q$ balls for $q \in (0,2]$.
They also analyze the performance of an estimator based on a multistage
thresholding procedure and show that asymptotically it nearly attains the optimal rate of
convergence.  Their analysis used spiked covariance
matrices (corresponding to $\lambda_2 \Sigma_0 = (\Id{p} -
\theta_1\theta_1^T)$ in \cref{eq:model} when $k=1$), while we allow a more
general class of covariance matrices. We note that our work provides
\emph{non-asymptotic} bounds that are optimal over $(p, n, \bar{R}_q)$ when 
$q \in \{0,1\}$ and optimal over $(p,n)$ when $q \in (0,1)$.

In next section, we present our main results along with some additional 
conditions to guarantee that estimation over $\modelclass{q}$ remains 
non-trivial. The main steps of the proofs are in \Cref{sec:proofs}. In the 
proofs we state some auxiliary lemmas. They are mainly technical, so we defer 
their proofs to the Appendix. 
\Cref{sec:conclusion} concludes the paper with some comments on extensions of this work.
\section{Main Results}
Our minimax results are formulated in terms of non-asymptotic bounds that
depend explicitly on $(n, p, R_q, \lambda_1,\lambda_2)$.  To facilitate
presentation, we introduce the notations
\begin{equation*}
  \bar{R}_q = R_q-1
  \text{ and }
  \sigma^2 = \frac{\lambda_1 \lambda_2}{(\lambda_1 - \lambda_2)^2}
  \,.
\end{equation*}
$\bar{R}_q$ appears naturally in our lower bounds because the eigenvector $\theta_1$ 
belongs to the sphere of dimension $p-1$ due to the constraint that $\norm{\theta_1}_2 = 1$.
Intuitively, $\sigma^2$ plays the role of the effective noise-to-signal ratio.  When 
comparing with minimax results for linear regression over $\ell_q$ balls, $\sigma^2$ is 
exactly analogous to the noise variance in the linear model.
Throughout the paper, there are absolute constants $c, C, c_1$, etc,\ldots
that may take different values in different expressions.

The following assumption on $R_q$, the size of the $\ell_q$ ball, is to ensure
that the eigenvector is not too dense.
\begin{assumption}
  \label{assumption:lq-radius}
  There exists $\alpha \in (0, 1]$, depending only on $q$, such that
  \begin{equation}
  \label{assumption:lq-radius-a}
    \bar{R}_q
    \leq
      \kappa^q (p-1)^{1-\alpha}
      \bar{R}_q^{\frac{2\alpha}{2-q}}
      \left[
        \frac{\sigma^2}{n} \log \frac{p-1}{\bar{R}_q^{\frac{2}{2-q}}}
      \right]^{\frac{q}{2}}
      \,,
  \end{equation}
  where $\kappa \leq c \alpha/16$ is a constant depending only on $q$, and
  \begin{equation}
  \label{assumption:lq-radius-b}
    1 \leq \bar{R}_q \leq e^{-1} (p-1)^{1-q/2}
    \,.
  \end{equation}
\end{assumption}
Assumption \autoref{assumption:lq-radius} also ensures that the effective
noise $\sigma^2$ is not too
small---this may happen if the spectral gap $\lambda_1 - \lambda_2$ is 
relatively large or if $\lambda_2$ is relatively close to 0.  In either case, 
the distribution of $X_i / \lambda_1^{1/2}$ would concentrate on a 1-dimensional subspace 
and the problem would effectively degrade into a low-dimensional one. 
If $R_q$ is relatively large, then $\Sphere_2^{p-1} \cap \Ball_q^p(R_q)$ is not much smaller than
$\Sphere_2^{p-1}$ and the parameter space will include many non-sparse vectors. 
In the case $q=0$, Assumption~\autoref{assumption:lq-radius} simplifies
because we may take $\alpha = 1$ and only require that
\begin{equation*}
  1 \leq \bar{R}_0 \leq e^{-1} (p-1)
  \,.
\end{equation*}
In the high-dimensional case that we are interested, where $p > n$, the
condition that
\begin{equation*}
 1 \leq \bar{R}_q \leq e^{-1} \kappa^q \sigma^q p^{(1-\alpha')/2}
 \,,
\end{equation*}
for some $\alpha' \in [0,1]$,
is sufficient to ensure that \eqref{assumption:lq-radius-a} holds for $q \in (0,1]$.
Alternatively, if we let $\alpha = 1 - q/2$ then \eqref{assumption:lq-radius-a} 
is satisfied for $q \in (0,1]$ if
\begin{equation*}
  1
  \leq
    \kappa^2 \sigma^2
    \big( (p-1)/n \big)
    \log\big( (p-1) / \bar{R}_q^{\frac{2}{2-q}} \big)
    \,.
\end{equation*}

The relationship between $n$, $p$, $R_q$ and $\sigma^2$ described in
Assumption \autoref{assumption:lq-radius} indicates a regime in which the
inference is neither impossible nor trivially easy. 
We can now state our first main result.

\begin{theorem}[Lower Bound for Sparse PCA]
  \label{thm:lower-bound}
  Let $q \in [0,1]$. If Assumption \autoref{assumption:lq-radius} holds, then 
  there exists a universal constant $c > 0$ depending only on $q$, 
  such that every estimator $\hat{\theta}_1$ satisfies
  \begin{align*}
    &
      \max_{ P \in \modelclass{q}(\lambda_1, \lambda_2,\bar{R}_q, \alpha, \kappa) } 
      \E_P\norm{\hat{\theta}_1\hat{\theta}_1^T - \theta_1\theta_1^T}_F
    \\
    &\geq
      c
      \min\left\{
        1 \,,\,
        \bar{R}_q^{\frac{1}{2}}
        \Biggr[
            \frac{\sigma^2}{n}
            \log\Big( (p-1) / \bar{R}_q^{\frac{2}{2-q}} \Big)
        \Biggr]^{\frac{1}{2}-\frac{q}{4}}
      \right\}
    \,.
  \end{align*}
\end{theorem}
Our proof of Theorem \ref{thm:lower-bound} is given in Section
\ref{proof:lower-bound}.  It follows the usual nonparametric lower
bound framework.  The main challenge is to construct a rich packing set in 
$\Sphere_2^{p-1}\cap \Ball_q^p(R_q)$.  (See \Cref{lem:special-set}.)
We note that a similar construction has been independently developed and
applied in similar a context by \textcite{Paul:draft}.

Our upper bound result is based on analyzing the solution to the
$\ell_q$-constrained maximization problem (\ref{eq:lqpca}), which
is a special case of empirical risk minimization.  In order to bound the
empirical process, we assume the data vector has sub-Gaussian tails, which is
nicely described by the Orlicz $\psi_\alpha$-norm.

\begin{definition}
  For a random variable $Y \in \Real$, the Orlicz $\psi_\alpha$-norm is
  defined for $\alpha \geq 1$ as
\begin{equation*}
  \norm{Y}_{\psi_\alpha}
  = \inf\{ c > 0 : \E \exp(|Y/c|^\alpha) \leq 2 \}
  \,.
\end{equation*}
\end{definition}
Random variables with finite $\psi_\alpha$-norm correspond to those whose
tails are bounded by $\exp(-Cx^\alpha)$. 

The case $\alpha = 2$ is important because it corresponds to random variables
with sub-Gaussian tails.  For example, if $Y \sim \normal(0, \sigma^2)$ then
$\norm{Y}_{\psi_2} \leq C \sigma$ for some positive constant $C$.
See \cite[Chapter 2.2]{vanderVaartAndWellner} for a complete introduction.

\begin{assumption}
  \label{assumption:subgaussian}
  There exist i.i.d. random vectors $Z_1,\ldots,Z_n \in \Real^p$ such that
  $\E Z_i = 0$, $\E Z_i Z_i^T = \Id{p}$,
  \begin{gather*}
    X_i = \mu + \Sigma^{1/2} Z_i \text{ and }
    \sup_{x \in \Sphere_2^{p-1}} \norm{\langle Z_i, x \rangle }_{\psi_2} \leq K
    \,,
  \end{gather*}
  where $\mu \in \Real^p$ and $K > 0$ is a constant.
\end{assumption}
Assumption~\autoref{assumption:subgaussian} holds for a variety of 
distributions, including the multivariate Gaussian (with $K^2 = 8/3$) and those 
of bounded random vectors. Under this assumption, we have the following theorem.

\begin{theorem}[Upper Bound for Sparse PCA]
  \label{thm:upper-bound}
  Let $\hat{\theta}_1$ be the $\ell_q$ constrained PCA estimate in \cref{eq:lqpca}
  with $\rho_q = R_q$, and let
  \begin{equation*}
    \epsilon = \norm{\hat{\theta}_1 \hat{\theta}_1^T - \theta_1 \theta_1^T}_F
    \text{ and }
    \tilde\sigma=\lambda_1/(\lambda_1-\lambda_2) 
    \,.
  \end{equation*}
  If the distribution of $(X_1,\ldots,X_n)$ belongs to 
  $\modelclass{q}(\lambda_1, \lambda_2,\bar{R}_q, \alpha, \kappa)$
  and satisfies 
  Assumptions~\autoref{assumption:lq-radius} and \autoref{assumption:subgaussian},
  then there exists a constant $c > 0$ depending only on $K$ such that the following hold:
  \begin{enumerate}
  \item If $q\in (0,1)$, then
  \begin{align*}
    \E \epsilon^2
    &\leq
      c
      \min\left\{ 
        1
      \,,\,
        R_q^2
        \Biggr[
          \frac{\tilde{\sigma}^2}{n} \log p
        \Biggr]^{1 - \frac{q}{2}}
      \right\}
    \,.
  \end{align*}
  \item If $q=1$, then
  \begin{align*}
    \E \epsilon^2
    &\leq
      c 
      \min\left\{ 
        1
      \,,\,
        R_1
        \Biggr[
          \frac{\tilde{\sigma}^2}{n} \log\Big(p / R_1^2 \Big)
        \Biggr]^{\frac{1}{2}}
      \right\}
    \,.
  \end{align*}
  \item If $q=0$, then
  \begin{align*}
      [\E \epsilon]^2
    &\leq 
      c
      \min\left\{ 
        1
      \,,\,
        R_0
        \frac{\tilde{\sigma}^2}{n} \log\Big(p / R_0 \Big)
      \right\}
    \,.
  \end{align*}
  \end{enumerate}
\end{theorem}

The proof of \Cref{thm:upper-bound} is given in 
\Cref{proof:upper-bound}. The different bounds for $q=0$, $q=1$, and $q\in
(0,1)$ are due to the different tools available for controlling empirical
processes in $\ell_q$ balls.
Comparing with \Cref{thm:lower-bound}, when
$q = 0$, the lower and upper bounds agree up to a factor
$\sqrt{\lambda_2/\lambda_1}$.  In the cases of $p=1$ and
$p\in(0,1)$, a lower bound in the squared error can be obtained by
using the fact $\E Y^2\ge (\E Y)^2$.  Therefore, 
over the class of distributions in 
$\modelclass{q}(\lambda_1, \lambda_2,\bar{R}_q, \alpha, \kappa)$ satisfying 
Assumptions~\autoref{assumption:lq-radius}~and~\autoref{assumption:subgaussian}, 
the upper and lower bound agree in terms of $(p,n)$ for all $q\in (0,1)$, 
and are sharp in $(p,n,R_q)$ for $q \in \{0,1\}$.

\section{Proofs of Main Results}
\label{sec:proofs}
We use the following notation in the proofs.
For matrices $A$ and $B$ whose dimensions are compatible, we define $\langle
A, B \rangle = \tr(A^T B)$. Then the Frobenius norm is $\norm{A}_F^2 = \langle
A, A \rangle$. The Kullback-Leibler (KL) divergence between two probability
measures $\prob_1,\prob_2$ is denoted by $\kldivergence{\prob_1}{\prob_2}$.

\subsection{Proof of the Lower Bound (\Cref{thm:lower-bound})}
\label{proof:lower-bound}
Our main tool for proving the minimax lower bound is the generalized Fano Method
\cite{Han:1994}.  The following version is from \cite[Lemma 3]{Yu:1997}.
\begin{lemma}[Generalized Fano method]
  \label{lem:fano}
  Let $N \geq 1$ be an integer and $\theta_1,\ldots,\theta_N \subset \Theta$ index a collection of
  probability measures $\prob_{\theta_i}$ on a measurable space $(\mathcal{X}, \mathcal{A})$.
  Let $d$ be a pseudometric on $\Theta$ and suppose that for all $i \neq j$
  \begin{equation*}
    d(\theta_i, \theta_j) \geq \alpha_N
  \end{equation*}
  and
  \begin{equation*}
    \kldivergence{\prob_{\theta_i}}{\prob_{\theta_j}} \leq \beta_N
    \,.
  \end{equation*}
  Then every $\mathcal{A}$-measurable estimator $\hat{\theta}$ satisfies
  \begin{equation*}
    \max_i \E_{\theta_i} d(\hat{\theta}, \theta_i)
    \geq \frac{\alpha_N}{2} \left(
      1 - \frac{\beta_N + \log 2}{\log N}
    \right)
    \,.
  \end{equation*}
\end{lemma}
The method works by converting the problem from estimation to
testing by discretizing the parameter space, and then applying Fano's
Inequality to the testing problem. (The $\beta_N$ term that appears above is
an upper bound on the mutual information.) 

To be successful, we must find a sufficiently large finite subset of the 
parameter space such that the points in the subset are $\alpha_N$-separated 
under the loss, yet nearly indistinguishable under the KL 
divergence of the corresponding probability measures.  We will use the subset 
given by the following lemma.
\begin{lemma}[Local packing set]
  \label{lem:special-set}
  Let $\bar{R}_q = R_q - 1 \geq 1$ and $p \geq 5$.
  There exists a finite subset 
  $\Theta_{\epsilon} \subset \Sphere_2^{p-1} \cap \Ball_q^p(R_q)$
  and an absolute constant $c > 0$ such that
  every distinct pair $\theta_1,\theta_2 \in \Theta_{\epsilon}$ satisfies
  \begin{equation*}
    \epsilon/\sqrt{2} < \norm{\theta_1 - \theta_2}_2 \leq \sqrt{2} \epsilon
    \,,
  \end{equation*}
  and
  \begin{align*}
      \log\card{\Theta_{\epsilon}}
    \geq
      c
      \left( \frac{\bar{R}_q}{\epsilon^q} \right)^{\frac{2}{2-q}}
      \left[
        \log (p-1)
        - \log \left( \frac{\bar{R}_q}{\epsilon^q} \right)^{\frac{2}{2-q}}
      \right]
  \end{align*}
  for all $q \in [0,1]$ and $\epsilon \in (0, 1]$.
\end{lemma}
Fix $\epsilon \in (0, 1]$ and let $\Theta_\epsilon$ denote the set given by
\Cref{lem:special-set}. With \Cref{lem:L2-equivalence} we have
\begin{equation}
  \label{eq:critical-separation}
      \epsilon^2 / 2
    \leq
      \norm{\theta_1\theta_1^T - \theta_2\theta_2^T}_F^2
    \leq
      4 \epsilon^2
\end{equation}
for all distinct pairs $\theta_1,\theta_2 \in \Theta_\epsilon$.
For each $\theta \in \Theta_\epsilon$, let
\begin{equation*}
  \Sigma_\theta = (\lambda_1 - \lambda_2) \theta \theta^T + \lambda_2 \Id{p}
  \,.
\end{equation*}
Clearly, $\Sigma_{\theta}$ has eigenvalues $\lambda_1 > \lambda_2 = \cdots =
\lambda_p$. Then $\Sigma_\theta$ satisfies
\cref{eq:model}. Let $\prob_\theta$ denote the $n$-fold product of the 
$\normal(0, \Sigma_\theta)$ probability measure. 
We use the following lemma to help bound the KL divergence.
\begin{lemma}
  \label{lem:kldivergence}
  For $i=1,2$, let $x_i \in \Sphere_2^{p-1}$, $\lambda_1 > \lambda_2  > 0$,
  \begin{equation*}
    \Sigma_i
    = (\lambda_1 - \lambda_2) x_i x_i^T + \lambda_2 \Id{p}
    \,,
  \end{equation*}
  and $\prob_i$ be the $n$-fold product of the $\normal(0,
  \Sigma_i)$ probability measure.
  Then
  \begin{equation*}
    \kldivergence{\prob_1}{\prob_2}
    =
    \frac{n}{2 \sigma^2} \norm{x_1 x_1^T - x_2 x_2^T}_F^2\,,
  \end{equation*}
  where $\sigma^2 = \lambda_1\lambda_2 / (\lambda_1 - \lambda_2)^2$.
\end{lemma}
Applying this lemma with \cref{eq:critical-separation} gives
\begin{align*}
    \kldivergence{\prob_{\theta_1}}{\prob_{\theta_2}}
  =
    \frac{n}{2\sigma^2} \norm{\theta_1\theta_1^T - \theta_2\theta_2^T}_F^2
  \leq
    \frac{2 n \epsilon^2}{\sigma^2}
  \,.
\end{align*}
Thus, we have found a subset of the parameter space that conforms to the requirements of
\Cref{lem:fano}, and so
\begin{equation*}
  \max_{\theta \in \Theta_\epsilon} \E_{\theta} \norm{\hat{\theta} \hat{\theta}^T - \theta \theta}_F
  \geq
  \frac{\epsilon}{2\sqrt{2}}
  \left(
    1 -
    \frac
    {2n\epsilon^2/\sigma^2 + \log 2}
    {\log \card{\Theta_\epsilon}}
  \right)
\end{equation*}
for all $\epsilon \in (0, 1]$.
The final step is to choose $\epsilon$ of the correct order.
If we can find $\epsilon$ so that
\begin{equation}
  \label{eq:fano-part-1}
  \frac{2n\epsilon^2/\sigma^2}{\log \card{\Theta_\epsilon}}
  \leq \frac{1}{4}
\end{equation}
and
\begin{equation}
  \label{eq:fano-part-2}
  \log \card{\Theta_\epsilon}
  \geq 4 \log 2
  \,,
\end{equation}
then we may conclude that
\begin{equation*}
  \max_{\theta \in \Theta_\epsilon} \E_{\theta} \norm{\hat{\theta} \hat{\theta}^T - \theta \theta}_F
  \geq
  \frac{\epsilon}{4\sqrt{2}}
  \,.
\end{equation*}
For a constant $C \in (0,1)$ to be chosen later, let
\begin{equation}\label{eq:choose_epsilon}
  \epsilon^2
  =
    \min\left\{
      1,
      C^{2-q}
      \bar{R}_q
      \left[
          \frac{\sigma^2}{n}
          \log \frac{p-1}{\bar{R}_q^{\frac{2}{2-q}}}
      \right]^{1-\frac{q}{2}}
    \right\}
  \,.
\end{equation}
We consider each of the two cases in the above $\min\{\cdots\}$ separately.
\paragraph{Case 1:}
Suppose that
\begin{equation}
  \label{eq:lower-bound-min-case-1}
    1
  \leq
    C^{2-q} \bar{R}_q
    \left[
        \frac{\sigma^2}{n}
        \log \frac{p-1}{\bar{R}_q^{\frac{2}{2-q}}}
    \right]^{1-\frac{q}{2}}
  \,.
\end{equation}
Then $\epsilon^2 = 1$ and by rearranging \eqref{eq:lower-bound-min-case-1}
\begin{equation*}
  \frac{n}{C^2 \sigma^2}
  \leq
  \bar{R}_q^{\frac{2}{2-q}}
  \left[
      \log(p-1)
      - \log \bar{R}_q^{\frac{2}{2-q}}
  \right]
  \,.
\end{equation*}
So by \Cref{lem:special-set},
\begin{align*}
    \log\card{\Theta_\epsilon}
  \geq
    c
    \bar{R}_q^{\frac{2}{2-q}}
    \left[
      \log(p-1) - \log \bar{R}_q^{\frac{2}{2-q}}
    \right]
  \geq 
    \frac{c n}{C^2 \sigma^2}
  \,.
\end{align*}
If we choose $C^2 \leq c / 16$, then
\begin{align*}
  \frac{2n\epsilon^2/\sigma^2}{\log \card{\Theta_\epsilon}}
  \leq
  \frac{4 C^2}{c}
  \leq \frac{1}{4}
  \,.
\end{align*}
To lower bound
\begin{align*}
    \log\card{\Theta_\epsilon}
  &\geq
    c
    \bar{R}_q^{\frac{2}{2-q}}
    \left[
      \log(p-1) - \log \bar{R}_q^{\frac{2}{2-q}}
    \right]
  \,,
\end{align*}
observe that the function $x \mapsto x \log[(p-1)/x]$ is increasing on $[1,(p-1)/e]$,
and, by Assumption~\autoref{assumption:lq-radius}, this interval contains $\bar{R}_q^{2/(2-q)}$.
If $p$ is large enough so that $p-1 \geq \exp\{(4/c)\log 2\}$, then
\begin{equation*}
  \log\card{\Theta_\epsilon}
  \geq c \log(p-1)
  \geq 4 \log 2
  \,.
\end{equation*}
Thus, \cref{eq:fano-part-1,eq:fano-part-2} are satisfied, and we conclude that
\begin{equation*}
  \max_{\theta \in \Theta_\epsilon} \E_{\theta} \norm{\hat{\theta} \hat{\theta}^T - \theta \theta}_F
  \geq
  \frac{\epsilon}{4\sqrt{2}}
  \,,
\end{equation*}
as long as $C^2 \leq c/16$ and $p-1 \geq \exp\{ (4/c) \log 2 \}$.

\paragraph{Case 2:}
Now let us suppose that
\begin{equation}\label{eq:case_2}
    1
  >
    C^{2-q} \bar{R}_q
    \left[
        \frac{\sigma^2}{n}
        \log \frac{p-1}{\bar{R}_q^{\frac{2}{2-q}}}
    \right]^{1-\frac{q}{2}}
  \,.
\end{equation}
Then
\begin{align}\label{eq:case_2_eps_identity}
  \left( \frac{\bar{R}_q}{\epsilon^q} \right)^{\frac{2}{2-q}}
  &=
    \frac{\bar{R}_q}{C^q}
    \left[
      \frac{\sigma^2}{n}
      \log \frac{p-1}{\bar{R}_q^{\frac{2}{2-q}}}
    \right]^{-\frac{q}{2}}
  \,,
\end{align}
and it is straightforward to check that Assumption~\autoref{assumption:lq-radius}
implies that if $C^q \geq \kappa^q$, then there is $\alpha \in (0,1]$, depending only on $q$,
such that
\begin{align}\label{eq:case_2_eps_ineq}
  \left( \frac{1}{\epsilon^q} \right)^{\frac{2}{2-q}}
  &\leq
  \left( \frac{p-1}{\bar{R}_q^{\frac{2}{2-q}}} \right)^{1-\alpha}
  \,.
\end{align}
So by \Cref{lem:special-set},
\begin{align}\label{eq:case_2_log_Theta_1}
  &
    \log\card{\Theta_\epsilon}\nonumber
  \\
  &\geq
    c
    \left( \frac{\bar{R}_q}{\epsilon^q} \right)^{\frac{2}{2-q}}
    \left[
      \log\frac{p-1}{\bar{R}_q^{\frac{2}{2-q}}}
      - \log \left( \frac{1}{\epsilon^q} \right)^{\frac{2}{2-q}}
    \right]\nonumber
  \\
  &\geq
  c \alpha \frac{\bar{R}_q}{C^q}
  \left[
    \frac{\sigma^2}{n}
    \log \frac{p-1}{\bar{R}_q^{\frac{2}{2-q}}}
  \right]^{-\frac{q}{2}}
  \left[ \log\frac{p-1}{\bar{R}_q^{\frac{2}{2-q}}} \right]
  \,,
\end{align}
where the last inequality is obtained by plugging in
(\ref{eq:case_2_eps_identity}) and (\ref{eq:case_2_eps_ineq}).

If we choose $C^2 \leq c\alpha/16$, then combining (\ref{eq:choose_epsilon}) and (\ref{eq:case_2_log_Theta_1}), we have
\begin{align}\label{eq:case_2_log_Theta_2}
    \frac{\frac{4 n \epsilon^2}{\sigma^2}}{\log\card{\Theta_\epsilon}}
  &\leq
    \frac{ 4 C^2 }{ c \alpha }
  \leq \frac{1}{4}
\end{align}
and \cref{eq:fano-part-1} is satisfied.
On the other hand, by (\ref{eq:case_2}) and the fact that $\bar R_q\ge 1$,
we have
\begin{equation*}
  C^{-q}\left[\frac{\sigma^2}{n}\log\frac{p-1}{\bar R_q^{\frac{2}{2-q}}}\right]^{-\frac{q}{2}}\ge 1\,,
\end{equation*}
and hence (\ref{eq:case_2_log_Theta_1}) becomes
\begin{align}\label{eq:case_2_log_Theta_3}
  \log\card{\Theta_\epsilon}
  &\geq
  c \alpha
  \bar{R}_q \left[ \log\frac{p-1}{\bar{R}_q^{\frac{2}{2-q}}} \right]
  \,.
\end{align}
The function $x \mapsto x \log[(p-1) / x^{2/(2-q)}]$ is increasing on
$[1, (p-1)^{1-q/2} / e]$ and,
by Assumption~\autoref{assumption:lq-radius}, $1 \leq \bar{R}_q \leq (p-1)^{1-q/2}/e$.
If $p-1 \geq \exp\{ [4/(c\alpha)] \log 2\}$,
then
\begin{align*}
  \log\card{\Theta_\epsilon}
  \geq
  c \alpha \log(p-1)
  \geq 4\log 2
\end{align*}
and \cref{eq:fano-part-2} is satisfied.  So we can conclude that
\begin{equation*}
  \max_{\theta \in \Theta_\epsilon} \E_{\theta} \norm{\hat{\theta} \hat{\theta}^T - \theta \theta}_F
  \geq
  \frac{\epsilon}{4\sqrt{2}}
  \,,
\end{equation*}
as long as $C^2 \leq c\alpha/16$ and $p-1 \geq \exp\{ [4/(c\alpha)] \log 2 \}$.

\paragraph{Cases 1 and 2 together:}
Looking back at cases 1 and 2, we see that because $\alpha \leq 1$, the
conditions that 
$\kappa^2 \leq C^2 \leq c \alpha / 16$
and 
$p-1 \geq \exp\{ [4/(c\alpha)] \log 2 \}$
are sufficient to ensure that
\begin{align*}
  &
    \max_{\theta \in \Theta_\epsilon} \E_{\theta} \norm{\hat{\theta} \hat{\theta}^T - \theta \theta}_F
  \\
  &\geq
  c'
  \min\left\{
    1,
    \bar{R}_q^{\frac{1}{2}}
    \left[
        \frac{\sigma^2}{n}
        \log \frac{p-1}{\bar{R}_q^{\frac{2}{2-q}}}
    \right]^{\frac{1}{2}-\frac{q}{4}}
  \right\}
  \,,
\end{align*}
for a constant $c' > 0$ depending only on $q$.
\qed

\subsection{Proof of the Upper Bound (\Cref{thm:upper-bound})}
\label{proof:upper-bound}
We begin with a lemma that bounds the curvature
of the matrix functional $\langle \Sigma, bb^T \rangle$.
\begin{lemma}
  \label{lem:curvature}
  Let $\theta \in \Sphere_2^{p-1}$. If $\Sigma \succeq 0$ has a unique largest 
  eigenvalue $\lambda_1$ with corresponding eigenvector $\theta_1$, then
  \begin{equation*}
    \frac{1}{2}
    (\lambda_1 - \lambda_2)
    \norm{\theta \theta^T - \theta_1 \theta_1^T}_F^2 
    \leq
    \langle
      \Sigma, \theta_1 \theta_1^T - \theta \theta^T
    \rangle
    \,.
  \end{equation*}
\end{lemma}
Now consider $\hat{\theta}_1$, the $\ell_q$-constrained sparse PCA estimator of $\theta_1$. Let $\epsilon = \norm{\hat{\theta}_1 \hat{\theta}_1^T - \theta_1 \theta_1^T}_F$.
Since $\theta_1 \in \Sphere_2^{p-1}$, it follows from Lemma
\ref{lem:curvature} that
\begin{align}
  \notag
    (\lambda_1 - \lambda_2)
    \epsilon^2 / 2
  \notag
  &\leq
    \langle \Sigma, \theta_1 \theta_1^T - \hat{\theta}_1 \hat{\theta}_1^T \rangle
  \\
  \notag
  &=
    \langle S, \theta_1 \theta_1^T \rangle
    -
    \langle \Sigma, \hat{\theta}_1 \hat{\theta}_1^T \rangle
    -
    \langle S - \Sigma, \theta_1 \theta_1^T \rangle
  \\
  \notag
  &\leq
    \langle S - \Sigma, \hat{\theta}_1 \hat{\theta}_1^T \rangle
    -
    \langle S - \Sigma, \theta_1 \theta_1^T \rangle
  \\
  \label{eq:L0-branch}
  &=
    \langle S - \Sigma, \hat{\theta}_1 \hat{\theta}_1^T - \theta_1 \theta_1^T \rangle
  \,.
\end{align}
We consider the cases $q \in (0,1)$, $q = 1$, and $q = 0$ separately.

\subsubsection{Case 1: $q \in (0, 1)$}
By applying H\"{o}lder's Inequality
to the right side of \cref{eq:L0-branch} and rearranging, we have
\begin{equation}
  \label{eq:L1-upperbound}
    \epsilon^2 / 2
  \leq
    \frac{\norm{\vec(S-\Sigma)}_\infty}{\lambda_1 - \lambda_2}
    \norm{\vec(\theta_1 \theta_1^T - \hat{\theta}_1\hat{\theta}_1^T )}_1
  \,,
\end{equation}
where $\vec(A)$ denotes the $1\times p^2$ matrix obtained by stacking the
columns of a $p\times p$ matrix $A$.
Since $\theta_1$ and $\hat{\theta}_1$ both belong to $\Ball_q^p(R_q)$,
\begin{align*}
    \norm{\vec(\theta_1 \theta_1^T - \hat{\theta}_1\hat{\theta}_1^T )}_q^q
  &\leq
    \norm{\vec(\theta_1 \theta_1^T)}_q^q + \norm{\vec(\hat{\theta}_1\hat{\theta}_1^T )}_q^q
  \\
  &\leq
    2 R_q^2
  \,.
\end{align*}
Let $t > 0$. We can use a standard truncation argument \cite[see,
e.g.,][Lemma 5]{Raskutti:2011} to show that
\begin{align*}
  &
    \norm{\vec(\theta_1 \theta_1^T - \hat{\theta}_1\hat{\theta}_1^T )}_1
  \\
  &\leq
    \sqrt{2} R_q \norm{\vec(\theta_1 \theta_1^T - \hat{\theta}_1\hat{\theta}_1^T)}_2 t^{-q/2}
    + 2 R_q^2 t^{1-q}
  \\
  &=
    \sqrt{2} R_q \norm{\theta_1 \theta_1^T - \hat{\theta}_1\hat{\theta}_1^T}_F t^{-q/2}
    + 2 R_q^2 t^{1-q}
  \\
  &=
    \sqrt{2} R_q \epsilon t^{-q/2} + 2 R_q^2 t^{1-q}
  \,.
\end{align*}
Letting $t = \norm{\vec(S-\Sigma)}_\infty / (\lambda_1 - \lambda_2)$ and joining with
\cref{eq:L1-upperbound} gives us
\begin{equation*}
  \label{eq:quadratic-upper-bound}
    \epsilon^2/2
  \leq
    \sqrt{2} t^{1-q/2} R_q \epsilon
    + 2 t^{2-q} R_q^2
  \,.
\end{equation*}
If we define $m$ implicitly so that $\epsilon = m \sqrt{2} t^{1-q/2} R_q$,
then the preceding inequality reduces to $m^2/2 \leq m+1$.  If $m \geq 3$,
then this is violated.  So we must have $m < 3$ and hence
\begin{align}
    \epsilon
  \leq 
    3 \sqrt{2} t^{1-q/2} R_q
  \label{eq:final-upper-bound}
  =
    3 \sqrt{2} R_q
    \left(\frac{\norm{\vec(S-\Sigma)}_\infty}{\lambda_1 - \lambda_2} \right)^{1 - q/2}
  \,.
\end{align}
Combining the above discussion with the sub-Gaussian assumption, the next lemma allows us to bound $\norm{\vec(S-\Sigma)}_\infty$.
\begin{lemma}
  \label{lem:Linf-deviation}
  If Assumption~\autoref{assumption:subgaussian} holds 
  and $\Sigma$ satisfies \eqref{eq:pop-spectraldecomposition}, then there is an absolute constant $c > 0$ such that
  \begin{equation*}
    \bignorm{ \norm{\vec(S - \Sigma)}_\infty }_{\psi_1}
    \leq c K^2 \lambda_1 \max\left\{\sqrt{\frac{\log p}{n}}, \frac{\log p}{n} \right\}
    \,.
  \end{equation*}
\end{lemma}
Applying \Cref{lem:Linf-deviation} to \cref{eq:final-upper-bound} gives
\begin{align*}
  &
      \norm{\epsilon^{2/(2-q)}}_{\psi_1}
  \\
  &\leq
      c R_q^{2/(2-q)} \frac{\bignorm{ \norm{\vec(S-\Sigma)}_\infty }_{\psi_1}}{\lambda_1 - \lambda_2}
  \\
  &\leq
      c K^2 R_q^{2/(2-q)} \tilde{\sigma}
      \max\left\{\sqrt{\frac{\log p}{n}}, \frac{\log p}{n} \right\}
  \,.
\end{align*}
The fact that $\E|X|^m \leq (m!)^m \norm{X}_{\psi_1}^m$
for $m \geq 1$ \cite[see][Chapter 2.2]{vanderVaartAndWellner} 
implies the following bound:
\begin{equation*}
    \E \epsilon^2 
  \leq
    c_K R_q^2 \tilde{\sigma}^{2-q}
    \max\left\{\sqrt{\frac{\log p}{n}}, \frac{\log p}{n} \right\}^{2-q}
  =:
    \mathbf{M}
    \,,
\end{equation*} 
Combining this with the trivial bound $\epsilon \leq 2$, yields
\begin{equation}
  \label{eq:trivial-bound}
  \E \epsilon^2 \leq \min(2, \mathbf{M})
  \,.
\end{equation}
If $\log p > n$, then $\E\epsilon^2 \leq 2$. Otherwise, we need only consider 
the square root term inside $\max\{\}$ in the definition of $\mathbf{M}$. 
Thus,
\begin{align*}
    \E\epsilon^2 
  &\leq c
  \min \left\{ 1 \,,\,
      R_q^2 \Biggr[ \frac{\tilde{\sigma}^2}{n} \log p \Biggr]^{1-\frac{q}{2}}
    \right\}
  \,.
\end{align*}
for an appropriate constant $c > 0$, depending only on $K$.
This completes the proof for the case $q \in (0, 1)$.

\subsubsection{Case 2: $q = 1$}
$\theta_1$ and $\hat{\theta}_1$ both belong to $\Ball_1^p(R_1)$.  So
applying the triangle inequality to the right side of \cref{eq:L0-branch} yields
\begin{align*}
  (\lambda_1 - \lambda_2)
  \epsilon^2 / 2
  &\leq
    \langle S - \Sigma, \hat{\theta}_1 \hat{\theta}_1^T - \theta_1 \theta_1^T \rangle
  \\
  &\leq
    |\hat{\theta}_1^T(S - \Sigma) \hat{\theta}_1|
    +
    |\theta_1^T(S - \Sigma) \theta_1|
  \\
  &\leq
    2 \sup_{b \in \Sphere_2^{p-1} \cap \Ball_1^p(R_1)}
    |b^T (S-\Sigma) b|
  \,.
\end{align*}
The next lemma provides a bound for the supremum. 
\begin{lemma} 
  \label{lem:L1-deviation}
  If Assumption~\autoref{assumption:subgaussian} holds 
  and $\Sigma$ satisfies \eqref{eq:pop-spectraldecomposition}, 
  then there is an absolute constant $c > 0$ such that
  \begin{align*}
    &
      \E \sup_{b \in \Sphere_2^{p-1} \cap \Ball_1^p(R_1)}
      | b^T(S - \Sigma) b |
    \\
    &\leq
      c \lambda_1 K^2
      \max \left\{
        R_1  \sqrt{ \frac{\log(p/R_1^2)}{n} }
      ,
        R_1^2 \frac{\log(p/R_1^2)}{n}
      \right\}
  \end{align*}
  for all $R_1^2 \in [1,p/e]$.
\end{lemma}
Assumption~\autoref{assumption:lq-radius} guarantees that $R_1^2 \in [1,p/e]$. 
Thus, we can apply Lemma~\autoref{lem:L1-deviation} and an argument similar to 
that used with \eqref{eq:trivial-bound} to complete the proof for the 
case $q=1$.

\subsubsection{Case 3: $q=0$}
We continue from \cref{eq:L0-branch}.  Since $\hat{\theta}_1$ and $\theta_1$
belong to $\Ball_0^p(R_0)$, their difference belongs to $\Ball_0^p(2 R_0)$.
Let $\Pi$ denote the diagonal matrix whose diagonal entries are 1 wherever
$\hat{\theta}_1$ or $\theta_1$ are nonzero, and 0 elsewhere. Then $\Pi$ has at
most $2 R_0$ nonzero diagonal entries, and $\Pi \hat{\theta}_1 =
\hat{\theta}_1$ and $\Pi \theta_1 = \theta_1$. So by the Von Neumann trace
inequality and \Cref{lem:canonical_angles},
\begin{align*}
  (\lambda_1 - \lambda_2) \epsilon^2 / 2
  &\leq
    | \langle S - \Sigma, \Pi(\hat{\theta}_1 \hat{\theta}_1^T - \theta_1 \theta_1^T)\Pi \rangle |
  \\
  &=
    | \langle \Pi(S - \Sigma)\Pi, \hat{\theta}_1 \hat{\theta}_1^T - \theta_1 \theta_1^T \rangle |
  \\
  &\leq
    \norm{\Pi(S - \Sigma)\Pi}_2
    \norm{ \hat{\theta}_1 \hat{\theta}_1^T - \theta_1 \theta_1^T }_{S_1}
  \\
  &=
    \norm{\Pi(S - \Sigma)\Pi}_2
    \sqrt{2} \epsilon
  \\
  &\leq
  \sup_{b \in \Sphere_2^{p-1} \cap \Ball_0^p(2 R_0)}
    |b^T (S - \Sigma) b|
    \sqrt{2} \epsilon
  \,,
\end{align*}
where $\norm{\,\cdot\,}_{S_1}$ denotes the sum of the singular values.
Divide both sides by $\epsilon$, rearrange terms, and then take the
expectation to get
\begin{equation*}
  \E \epsilon \leq
  \frac{c}{\lambda_1 - \lambda_2}
  \E \sup_{b \in \Sphere_2^{p-1} \cap \Ball_0^p(2 R_0)} |b^T (S - \Sigma) b|
  \,.
\end{equation*}

\begin{lemma}
  \label{lem:L0-Gaussian}
  If Assumption~\autoref{assumption:subgaussian} holds and $\Sigma$ satisfies
  \eqref{eq:pop-spectraldecomposition}, then there is an absolute constant $c > 0$ such that
  \begin{align*}
    &
      \E \sup_{b \in \Sphere_2^{p-1} \cap \Ball_0^p(d)} |b^T (S - \Sigma) b|
    \\
    &\leq c K^2 \lambda_1
      \max\left\{
        \sqrt{(d/n) \log(p/d)}
        ,
        (d/n) \log(p/d)
      \right\}
  \end{align*}
  for all integers $d \in [1,p/2)$.
\end{lemma}
Taking $d = 2 R_0$ and applying an argument similar to 
that used with \eqref{eq:trivial-bound} completes the proof of the $q = 0$ case.
\qed

\section{Conclusion and Further Extensions}
\label{sec:conclusion}
We have presented upper and lower bounds on the minimax estimation error for
sparse PCA over $\ell_q$ balls. The bounds are sharp in $(p,n)$, and they show
that $\ell_q$ constraints on the leading eigenvector make estimation possible
in high-dimensions even when the number of variables greatly exceeds the
sample size. Although we have specialized to the case $k = 1$ (for the leading
eigenvector), our methods and arguments can be extended to the
multi-dimensional subspace case ($k > 1$). One nuance in that case is that
there are different ways to generalize the notion of $\ell_q$ sparsity to
multiple eigenvectors. A potential difficulty there is that if there is
multiplicity in the eigenvalues or if eigenvalues coalesce, then the
eigenvectors need not be unique (up to sign). So care must be taken to handle
this possibility.

\subsubsection*{Acknowledgements}
V.~Q.~Vu was supported by a NSF Mathematical Sciences Postdoctoral Fellowship
(DMS-0903120). J.~Lei was supported by NSF Grant BCS0941518. We thank the 
anonymous reviewers for their helpful comments.

\printbibliography

\clearpage
\appendix
\section{APPENDIX - SUPPLEMENTARY MATERIAL}

\subsection{Additional Technical Tools}
We state below two results that we use frequently in our proofs.
The first is well-known consequence of the CS decomposition. It relates the
canonical angles between subspaces to the singular values of products and
differences of their corresponding projection matrices.
\begin{lemma}[{\textcite[Theorem I.5.5]{StewartAndSun}}]
  \label{lem:canonical_angles}
  Let $\mathcal{X}$ and $\mathcal{Y}$ be $k$-dimensional subspaces of $\Real^p$ 
  with orthogonal projections $\Pi_\mathcal{X}$ and $\Pi_\mathcal{Y}$.
  Let $\sigma_1 \geq \sigma_2 \geq \cdots \geq \sigma_k$ be the sines of the
  canonical angles between $\mathcal{X}$ and $\mathcal{Y}$.
  Then
  \begin{enumerate}
    \item The singular values of $\Pi_\mathcal{X}( \Id{p} - \Pi_\mathcal{Y})$ are
    $$\sigma_1,\sigma_2,\ldots,\sigma_k, 0, \ldots, 0\,.$$
    \item The singular values of $\Pi_\mathcal{X} - \Pi_\mathcal{Y}$ are
    $$\sigma_1,\sigma_1, \sigma_2, \sigma_2,\ldots,\sigma_k, \sigma_k, 0, \ldots, 0\,.$$
  \end{enumerate}
\end{lemma}

\begin{lemma}
  \label{lem:L2-equivalence}
  Let $x,y \in \Sphere_2^{p-1}$.  Then
  \begin{equation*}
    \norm{xx^T - yy^T}_F^2 \leq 2 \norm{x - y}_2^2
  \end{equation*}
  If in addition $\norm{x - y}_2 \leq \sqrt{2}$, then 
  \begin{equation*}
    \norm{xx^T - yy^T}_F^2 \geq \norm{x - y}_2^2
  \end{equation*}
\end{lemma}
\begin{proof}
  By \Cref{lem:canonical_angles} and the polarization identity
  \begin{align*}
    \frac{1}{2} \norm{xx^T - yy^T}_F^2 
    &= 1 - (x^T y)^2 \\
    &= 1 - \left(\frac{2 - \norm{x-y}^2}{2}\right)^2 \\
    &= \norm{x-y}_2^2 - \norm{x-y}_2^4/4 \\
    &= \norm{x-y}_2^2(1 - \norm{x-y}_2^2/4)
    \,.
  \end{align*}
  The upper bound follows immediately.  Now if $\norm{x - y}_2^2 \leq 2$, then
  the above right-hand side is bounded from below by $\norm{x - y}_2^2/2$.
\end{proof}

\subsection{Proofs for \Cref{thm:lower-bound}}
\paragraph{Proof of \Cref{lem:special-set}}
Our construction is based on a hypercube argument.
We require a variation of the Varshamov-Gilbert bound due to
\textcite{Birge:1998}. We use a specialization of the version that appears 
in \cite[Lemma 4.10]{Massart:2007}.
\begin{lemma*}
  \label{lem:massart}
  Let $d$ be an integer satisfying $1 \leq d \leq (p-1)/4$.  There exists a subset 
  $\Omega_d \subset \{0,1\}^{p-1}$ that satisfies the following
  properties:
  \begin{enumerate}
    \item $\norm{\omega}_0 = d$ for all $\omega \in \Omega_d$,
    \item $\norm{\omega - \omega'}_0 > d/2$ for all distinct pairs $\omega, \omega' \in \Omega_d$, and
    \item $\log\card{\Omega_d} \geq c d \log((p-1)/d)$, where $c \geq 0.233$.
  \end{enumerate}
\end{lemma*}
Let $d \in [1, (p-1)/4]$ be an integer, $\Omega_d$ be the corresponding
subset of $\{0,1\}^{p-1}$ given by preceding lemma,
\begin{equation*}
  x(\omega) 
  = \big( (1- \epsilon^2)^{\frac{1}{2}}, \epsilon \omega d^{-\frac{1}{2}} \big)
  \in \Real^p
  \,,
\end{equation*}
and
\begin{equation*}
  \Theta = \{ x(\omega) : \omega \in \Omega_d \}\,.    
\end{equation*}
Clearly, $\Theta$ satisfies the following properties:
\begin{enumerate}
  \item \label{it:special-set-sphere}
  $\Theta \subseteq \Sphere_2^{p-1}$,
  \item \label{it:special-set-separation}
  $\epsilon / \sqrt{2} < \norm{\theta_1 - \theta_2}_2 \leq \sqrt{2} \epsilon$ for all distinct pairs $\theta_1,\theta_2 \in \Theta_d$, 
  \item \label{it:special-set-lq-norm}
  $\norm{\theta}_q^q \leq 1 + \epsilon^q d^{(2-q)/2}$ for all $\theta \in \Theta$, and
  \item \label{it:special-set-cardinality}
  $\log\card{\Theta} \geq c d [ \log(p-1) - \log d ]$, where $c \geq 0.233$.
\end{enumerate}
To ensure that $\Theta$ is also contained in $\Ball_q^p(R_q)$, we will choose $d$ so that 
the right side of the upper bound in \cref{it:special-set-lq-norm} is smaller than $R_q$.
Choose 
\begin{equation*}
  d 
  = 
  \left\lfloor 
    \min\Big\{
      (p-1)/4 ,
      \left( \bar{R}_q / \epsilon^q \right)^{\frac{2}{2-q}} 
    \Big\}
  \right\rfloor
  \,.
\end{equation*}
The assumptions that $p \geq 5$, $\epsilon \leq 1$, and $\bar{R}_q
\geq 1$ guarantee that this is a valid choice satisfying $d \in [1, (p-1)/4]$.
The choice also guarantees that $\Theta \subset \Ball_q^p(R_q)$, because
\begin{align*}
  \norm{\theta}_q^q 
    &\leq 
    1 + \epsilon^q d^{(2-q)/2} \\
    &\leq 
    1 + \epsilon^q \left( \bar{R}_q / \epsilon^q \right) 
    = 
    R_q
\end{align*}
for all $\theta \in \Theta$. 
To complete the proof we will show that $\log\card{\Theta}$ satisfies the lower bound
claimed by the lemma. Note that the function $a \mapsto a \log[(p-1)/a]$ is
increasing on $[0, (p-1)/e]$ and decreasing on $[(p-1)/e, \infty)$. So if
\begin{equation*}
  a :=
  \left( \frac{\bar{R}_q}{\epsilon^q} \right)^{\frac{2}{2-q}}
  \leq \frac{p-1}{4}
  \,,
\end{equation*}
then
\begin{align*}
    \log\card{\Theta} 
  &\geq
    c d \left[ \log (p-1) - \log d \right]
  \\
  &\geq 
    (c/2) a \left[ \log (p-1) - \log a \right]
  \,,
\end{align*}
because $d = \lfloor a \rfloor \geq a/2$.
Moreover, since $d \leq (p-1)/4$ and the above right hand side is maximized when $a
= (p-1)/e$, the inequality remains valid for all $a \geq 0$ if we replace the
constant $(c/2)$ with the constant
\begin{align*}
  c'
  &= (c/2) \frac{ \frac{p-1}{4}[\log(p-1) - \log \frac{p-1}{4}] }
           { \frac{p-1}{e}[\log(p-1) - \log \frac{p-1}{e}] }
  \\
  &= (c/2) \frac{e \log 4}{4} \geq 0.109
  \,.
\end{align*}
\qed
\paragraph{Proof of \Cref{lem:kldivergence}}
Let $A_i = x_i x_i^T$ for $i =1,2$.  
Then $\Sigma_i = \lambda_1 A_i + \lambda_2 ( \Id{p} - A_i )$.
Since $\Sigma_1$ and $\Sigma_2$ have the same eigenvalues 
and hence the same determinant,
\begin{align*}
  \kldivergence{\prob_1}{\prob_2}
  &= 
    \frac{n}{2} \left[
      \tr(\Sigma_2^{-1} \Sigma_1) - p
      - \log\det(\Sigma_2^{-1} \Sigma_1)
    \right]
  \\
  &=
    \frac{n}{2} \left[
      \tr(\Sigma_2^{-1} \Sigma_1) - p
    \right]
  \\
  &=
    \frac{n}{2} \tr(\Sigma_2^{-1} (\Sigma_1 - \Sigma_2) )
    \,.
\end{align*}
The spectral decomposition 
$\Sigma_2 = \lambda_1 A_2 + \lambda_2 (\Id{p} - A_2)$
allows us to easily calculate that 
\begin{equation*}
  \Sigma_2^{-1}
  = 
  \lambda_2^{-1} (\Id{p} - A_2) + \lambda_1^{-1} A_2
  \,.
\end{equation*}
Since orthogonal projections are idempotent, i.e. $A_i A_i = A_i$,
\begin{align*}
  &
    \Sigma_2^{-1} (\Sigma_1 - \Sigma_2)
  \\
  &=
    \frac{\lambda_1 - \lambda_2}{\lambda_1}
    [ (\lambda_1 / \lambda_2) (\Id{p} - A_2) + A_2 ]
    (A_1 - A_2)
  \\
  &=
    \frac{\lambda_1 - \lambda_2}{\lambda_1}
    [ (\lambda_1 / \lambda_2) (\Id{p} - A_2) A_1 -  A_2 (A_2 - A_1) ]
  \\
  &=
    \frac{\lambda_1 - \lambda_2}{\lambda_1}
    [ (\lambda_1 / \lambda_2) (\Id{p} - A_2) A_1 -  A_2 (\Id{p} - A_1) ]
  \,.
\end{align*}
Using again the idempotent property and symmetry of projection matrices,
\begin{align*}
  &
    \tr((\Id{p} - A_2) A_1 )
  \\
  &=
    \tr( (\Id{p} - A_2) (\Id{p} - A_2) A_1 A_1)
  \\
  &=
    \tr(A_1 (\Id{p} - A_2) (\Id{p} - A_2) A_1)
  \\
  &=
    \norm{A_1 (\Id{p} - A_2)}_F^2
\end{align*}
and similarly,
\begin{equation*}
  \tr(A_2 (\Id{p} - A_1)) = \norm{A_2 (\Id{p} - A_1)}_F^2
  \,.
\end{equation*}
By \Cref{lem:canonical_angles}, 
\begin{equation*}
  \norm{A_1 (\Id{p} - A_2)}_F^2
  = \norm{A_2 (\Id{p} - A_1)}_F^2
  = \frac{1}{2} \norm{A_1 - A_2}_F^2
  \,.
\end{equation*}
Thus,
\begin{align*}
    \tr( \Sigma_2^{-1} (\Sigma_1 - \Sigma_2) )
  &=
    \frac{(\lambda_1 - \lambda_2)^2}{2 \lambda_1 \lambda_2}
    \norm{A_1 - A_2}_F^2
\end{align*}
and the result follows.
\qed

\subsection{Proofs for \Cref{thm:upper-bound}}
\paragraph{Proof of \Cref{lem:curvature}}
We begin with the expansion,
\begin{align*}
  &\langle \Sigma, \theta_1\theta_1^T - \theta\theta^T \rangle
  \\
  &=
  \tr\{\Sigma \theta_1 \theta_1^T\} - \tr\{\Sigma \theta \theta^T\} \\
  &=
  \tr\{\Sigma(\Id{p} - \theta\theta^T) \theta_1 \theta_1^T\}
  - \tr\{\Sigma \theta \theta^T (\Id{p} - \theta_1\theta_1^T) \}
  \,.
\end{align*}
Since $\theta_1$ is an eigenvector of $\Sigma$ corresponding to
the eigenvalue $\lambda_1$,
\begin{align*}
  &\tr\{\Sigma(\Id{p} - \theta\theta^T) \theta_1 \theta_1^T\} \\
  &=
  \tr\{\theta_1 \theta_1^T \Sigma(\Id{p} - \theta\theta^T) \theta_1 \theta_1^T\} \\
  &=
  \lambda_1 \tr\{\theta_1 \theta_1^T (\Id{p} - \theta\theta^T) \theta_1 \theta_1^T\} \\
  &=
  \lambda_1 \tr\{\theta_1 \theta_1^T (\Id{p} - \theta\theta^T)^2 \theta_1 \theta_1^T\} \\
  &=
  \lambda_1 \norm{\theta_1 \theta_1^T (\Id{p} - \theta\theta^T)}_F^2
  \,.
\end{align*}
Similarly, we have
\begin{align*}
  &
    \tr\{\Sigma \theta \theta^T (\Id{p} - \theta_1\theta_1^T) \} \\
  &=
    \tr\{(\Id{p} - \theta_1\theta_1^T) \Sigma \theta \theta^T (\Id{p} - \theta_1\theta_1^T) \} \\
  &=
    \tr\{\theta^T (\Id{p} - \theta_1\theta_1^T) \Sigma (\Id{p} - \theta_1\theta_1^T) \theta \} \\
  &\leq
    \lambda_2
    \tr\{\theta^T (\Id{p} - \theta_1\theta_1^T)^2 \theta \} \\
  &=
    \lambda_2
    \norm{\theta \theta^T(\Id{p} - \theta_1\theta_1^T)}_F^2
  \,.
\end{align*}
Thus,
\begin{align*}
    \langle \Sigma, \theta_1\theta_1^T - \theta\theta^T \rangle
  &\geq
    (\lambda_1 - \lambda_2) \norm{\theta \theta^T(\Id{p} - \theta_1\theta_1^T)}_F^2
    \\
  &= 
    \frac{1}{2} (\lambda_1 - \lambda_2) \norm{\theta \theta^T - \theta_1\theta_1^T}_F^2
  \,.
\end{align*}
The last inequality follows from \Cref{lem:canonical_angles}.
\qed

\paragraph{Proof of \Cref{lem:Linf-deviation}}
Since the distribution of $S - \Sigma$ does not depend on $\mu = \E X_i$,
we assume without loss of generality that $\mu = 0$.
Let $a,b \in \{1,\ldots,p\}$ and
\begin{align*}
  D_{ab}
  &=
    \frac{1}{n} \sum_{i=1}^n (X_m)_a (X_m)_b  - \Sigma_{ab}
  \\
  &=:
    \frac{1}{n} \sum_{i=1}^n \zeta_i - \E \zeta_i
    \,.
\end{align*}
Then
\begin{equation*}
  (S - \Sigma)_{ab} = D_{ab} - \bar{X}_a \bar{X}_b
  \,.
\end{equation*}
Using the elementary inequality $2 |ab| \leq a^2 + b^2$, we have
by Assumption~\autoref{assumption:subgaussian} that
\begin{align*}
    \norm{\zeta_i}_{\psi_1}
  &=
    \norm{\langle X_i, 1_a\rangle \langle X_i, 1_b \rangle }_{\psi_1}
  \\
  &\leq
    \max_a \norm{|\langle X_i, 1_a\rangle|^2}_{\psi_1}
  \\
  &\leq
    2 \max_a \norm{\langle \Sigma^{1/2} Z_i, 1_a\rangle}_{\psi_2}^2
  \\
  &\leq
    2 \lambda_1 K^2
  \,.
\end{align*}
In the third line, we used the fact that the $\psi_1$-norm is bounded above by
a constant times the $\psi_2$-norm \cite[see][p. 95]{vanderVaartAndWellner}.
By a generalization of Bernstein's Inequality for the $\psi_1$-norm
\cite[see][Section 2.2]{vanderVaartAndWellner}\,,
for all $t > 0$
\begin{align*}
    \prob(|D_{ab}| > 8 t \lambda_1 K^2)
  &\leq
    \prob(|(D_{ab}| > 4 t \norm{\zeta_i}_{\psi_1})
  \\
  &\leq
  2 \exp( -n \min\{t, t^2\} / 2 )
  \,.
\end{align*}
This implies \cite[Lemma 2.2.10]{vanderVaartAndWellner} the bound
\begin{equation}
  \begin{aligned}
    \label{eq:d-bound}
    &
      \bignorm{ \max_{ab} |D_{ab}| }_{\psi_1}
    \\
    &\leq
      c K^2 \lambda_1
      \max\left\{ \sqrt{\frac{\log p}{n}}, \frac{\log p}{n} \right\}
    \,.
  \end{aligned}
\end{equation}
Similarly,
\begin{align*}
  2 \norm{\bar{X}_a \bar{X}_b}_{\psi_1}
  &\leq
    \norm{|\langle \bar{X}, 1_a\rangle|^2}_{\psi_1} + \norm{|\langle \bar{X}, 1_b\rangle|^2}_{\psi_1}
  \\
  &\leq
    \norm{\langle \bar{X}, 1_a\rangle}_{\psi_2}^2 + \norm{\langle \bar{X}, 1_b\rangle}_{\psi_2}^2
  \\
  &\leq
    \frac{2}{n^2} \sum_{i=1}^n \norm{\langle X_i, 1_a\rangle}_{\psi_2}^2 + \norm{\langle X_i, 1_b\rangle}_{\psi_2}^2
  \\
  &\leq \frac{4}{n} \lambda_1 K^2
  \,.
\end{align*}
So by a union bound \cite[Lemma 2.2.2]{vanderVaartAndWellner},
\begin{equation}
  \label{eq:xbar-bound}
  \bignorm{\max_{ab} |\bar{X}_a \bar{X}_b|}_{\psi_1}
  \leq
  c K^2 \lambda_1 \frac{\log p}{n}
  \,.
\end{equation}
Adding \cref{eq:d-bound,eq:xbar-bound} and then adjusting the constant $c$
gives the desired result, because
\begin{align*}
  &
    \bignorm{\norm{\vec(S-\Sigma)}_\infty}_{\psi_1}
  \\
  &\leq
    \bignorm{\max_{ab} |D_{ab}|}_{\psi_1}
    + \bignorm{\max_{ab} | \bar{X}_a \bar{X}_b |}_{\psi_1}
  \,.
\end{align*}
\qed

\paragraph{Proof of \Cref{lem:L1-deviation}}
Let $B = \Sphere_2^{p-1} \cap \Ball_1^p(R_1)$.
We will use a recent result in empirical process theory due to \textcite{Mendelson:2010}
to bound
\begin{equation*}
  \sup_{b \in B}
  b^T(S - \Sigma) b
\,.
\end{equation*}
The result uses \citeauthor{Talagrand}'s generic chaining method,
and allows us to reduce the problem to bounding the
supremum of a Gaussian process. The statement of the result involves
the generic chaining complexity, $\gamma_2(B,d)$, of a set $B$ equipped with the metric $d$.
We only use a special case, $\gamma_2(B, \norm{\,\cdot\,}_2)$,
where the complexity measure is equivalent to the expectation of the supremum
of a Gaussian process on $B$. We refer the reader to \cite{Talagrand} for
a complete introduction.
\begin{lemma}[{\textcite{Mendelson:2010}}]
  \label{lem:mendelson}
  Let $Z_i$, $i=1,\ldots,n$ be i.i.d. random variables.
  There exists an absolute constant $c$ for which the following holds.
  If $\mathcal{F}$ is a symmetric class of mean-zero functions then
  \begin{align*}
    &
      \E \sup_{f \in \mathcal{F}}
      \left|
      \frac{1}{n} \sum_{i=1}^n f^2(Z_i) - \E f^2(Z_i)
      \right|
    \\
    &\leq
      c \max\left\{
        d_{\psi_1} \frac{\gamma_2(\mathcal{F}, \psi_2)}{\sqrt{n}}
        \,,
        \frac{\gamma_2^2(\mathcal{F}, \psi_2)}{n}
        \right\}
    \,,
  \end{align*}
  where $d_{\psi_1} = \sup_{f \in \mathcal{F}} \norm{f}_{\psi_1}$.
\end{lemma}

Since the distribution of $S - \Sigma$ does not depend on
$\mu = \E X_i$, we assume without loss of generality that $\mu = 0$.
Then $| b^T(S - \Sigma) b |$ is bounded from above
by a sum of two terms,
\begin{equation*}
  \left| b^T \left(\frac{1}{n} \sum_{i=1}^n X_i X_i^T - \Sigma\right) b \right|
  + b^T \bar{X} \bar{X}^T b
  \,,
\end{equation*}
which can be rewritten as
\begin{equation*}
  D_1(b)
  :=
  \left|
  \frac{1}{n} \sum_{i=1}^n
  \langle Z_i, \Sigma^{1/2} b \rangle^2 - \E \langle Z_i, \Sigma^{1/2} b \rangle^2
  \right|
\end{equation*}
and $D_2(b) := \langle \bar{Z}, \Sigma^{1/2} b \rangle^2$,
respectively.  To apply \Cref{lem:mendelson} to $D_1$,
define the class of linear functionals
\begin{equation*}
  \mathcal{F}
  := \{
    \langle \cdot, \Sigma^{1/2} b\rangle
    :
      b \in B
  \}
  \,.
\end{equation*}
Then
\begin{equation*}
  \sup_{b \in B}
  D_1(b) =
  \sup_{f \in \mathcal{F}}
  \left|
  \frac{1}{n} \sum_{i=1}^n f^2(Z_i) - \E f^2(Z_i)
  \right|
  \,,
\end{equation*}
and we are in the setting of Lemma \ref{lem:mendelson}.

First, we bound the $\psi_1$-diameter of $\mathcal{F}$.
\begin{align*}
  d_{\psi_1}
  &=
    \sup_{ b \in B }
    \norm{\langle Z_i, \Sigma^{1/2} b \rangle}_{\psi_1}
  \\
  &\leq c
  \sup_{ b \in B }
    \norm{\langle Z_i, \Sigma^{1/2} b \rangle}_{\psi_2}
  \,.
\end{align*}
By Assumption~\autoref{assumption:subgaussian},
\begin{equation*}
  \norm{\langle Z_i, \Sigma^{1/2} b \rangle}_{\psi_2}
  \leq K \norm{\Sigma^{1/2} b}_2
  \leq K \lambda_1^{1/2}
\end{equation*}
and so
\begin{equation}
  \label{eq:psi1-diameter}
    d_{\psi_1}
  \leq
    c K \lambda_1^{1/2}
  \,.
\end{equation}
Next, we bound $\gamma_2(\mathcal{F}, \psi_2)$ by showing that the
metric induced by the $\psi_2$-norm on $\mathcal{F}$ is equivalent
to the Euclidean metric on $B$.
This will allow us to reduce the problem to bounding the supremum of a
Gaussian process.
For any $f,g \in \mathcal{F}$,
by Assumption~\autoref{assumption:subgaussian},
\begin{align}
  \notag
  \norm{(f - g)(Z_i)}_{\psi_2}
  &=
    \norm{\langle Z_i, \Sigma^{1/2} (b_f - b_g) \rangle }_{\psi_2}
  \\
  \notag
  &\leq
    K \norm{\Sigma^{1/2} (b_f - b_g)}_2
  \\
  \label{eq:psi2-bound}
  &\leq
    K \lambda_1^{1/2} \norm{b_f - b_g}_2
  \,,
\end{align}
where $b_f,b_g \in B$.
Thus, by \cite[Theorem 1.3.6]{Talagrand},
\begin{align*}
    \gamma_2(\mathcal{F}, \psi_2)
  &\leq
    c K \lambda_1^{1/2} \gamma_2 ( B, \norm{\,\cdot\,}_2 )
  \,.
\end{align*}
Then applying Talagrand's Majorizing Measure Theorem \cite[Theorem 2.1.1]{Talagrand}
yields
\begin{equation}
  \label{eq:gamma2-complexity}
    \gamma_2(\mathcal{F}, \psi_2)
  \leq
    c K \lambda_1^{1/2} \E
    \sup_{b \in B}
    \langle Y, b \rangle
  \,,
\end{equation}
where $Y$ is a $p$-dimensional standard Gaussian random vector.
Recall that $B = \Ball_1^p(R_1) \cap \Sphere_2^{p-1}$. So 
\begin{align*}
    \E \sup_{b \in B} \langle Y, b \rangle
  &\leq
    \E \sup_{b \in \Ball_1^p(R_1) \cap \Ball_2^p(1)} 
    \langle Y, b \rangle
  \,.
\end{align*}
Here, we could easily upper bound the above quantity by the supremum over 
$\Ball_1^p(R_q)$ alone. Instead, we use a sharper upper bound due to 
\textcite[Theorem 5.1]{Gordon:2007}:
\begin{align*}
    \E \sup_{b \in \Ball_1^p(R_1) \cap \Ball_2^p(1)} 
    \langle Y, b \rangle  
  &\leq
    R_1 \sqrt{2 + \log(2 p / R_1^2)}
  \\
  &\leq
    2 R_1 \sqrt{\log(p/R_1^2)}
  \,,
\end{align*}
where we used the assumption that $R_1^2 \leq p/e$ in the last inequality.
Now we apply \Cref{lem:mendelson} to get
\begin{align*}
  &
    \E \sup_{b \in B} D_1(B)
  \\
  &\leq
    c K^2 \lambda_1
    \max \left\{
      R_1 \sqrt{ \frac{\log(p/R_1^2)}{n} }
    ,
      R_1^2 \frac{\log(p/R_1^2)}{n}
    \right\}
  \,.
\end{align*}
Turning to $D_2(b)$, we can take $n = 1$ in \Cref{lem:mendelson} and
use a similar argument as above, because
\begin{equation*}
  D_2(b)
  \leq |\langle \bar{Z}, \Sigma^{1/2} b \rangle^2 - \E \langle \bar{Z}, \Sigma^{1/2} b \rangle^2|
  + \E \langle \bar{Z}, \Sigma^{1/2} b \rangle^2
  \,.
\end{equation*}
We just need to bound the $\psi_2$-norms of
$f(\bar{Z})$ and $(f-g)(\bar{Z})$ to get bounds that are analogous to
\cref{eq:psi1-diameter,eq:psi2-bound}.
Since $\bar{Z}$ is the sum of the independent random variables
$Z_i/n$,
\begin{align*}
    \sup_{b \in B}
    \norm{f(\bar{Z})}_{\psi_2}^2
  &=
    \sup_{b \in B}
    \norm{\langle \bar{Z}, \Sigma^{1/2} b_f \rangle}_{\psi_2}^2
  \\
  &\leq
    \sup_{b \in B}
    c \sum_{i=1}^n \norm{\langle Z_i, \Sigma^{1/2} b_f \rangle}_{\psi_2}^2 / n^2
  \\
  &\leq
    \sup_{b \in B}
    c K^2 \lambda_1 \norm{b_f}_2^2 / n
  \\
  &\leq
    c K^2 \lambda_1 / n
  \,,
\end{align*}
and similarly,
\begin{equation*}
  \norm{(f-g)(\bar{Z})}_{\psi_2}
  \leq c K \lambda_1 \norm{b_f - b_g}_2^2 / n
  \,.
\end{equation*}
So repeating the same arguments as for $D_1$, we get a similar bound for $D_2$.
Finally, we bound $\E D_2(b)$ by
\begin{align*}
  \E \langle \bar{X}, b \rangle^2
  &= b^T \Big( \sum_{i=1}^n \sum_{j=1}^n \E X_i X_j^T / n^2 \Big) b
  \\
  &= b^T \Big( \sum_{i=1}^n \E X_i X_i^T / n^2 \Big) b
  \\
  &= \norm{\Sigma^{1/2} b}_2^2 / n \\
  &\leq \lambda_1 / n
  \,.
\end{align*}
Putting together the bounds for $D_1$ and $D_2$ and then adjusting constants
completes the proof.
\qed

\paragraph{Proof of \Cref{lem:L0-Gaussian}}
Using a similar argument as in the proof of \Cref{lem:L1-deviation} we can show that
\begin{align*}
  \E \sup_{b \in \Sphere_2^{p-1} \cap \Ball_0^p(d)} 
  |b^T (S - \Sigma) b|
  \leq c K^2 \lambda_1
  \max \left\{ \frac{A}{\sqrt{n}}, \frac{A^2}{n} \right\}
  \,,
\end{align*}
where 
\begin{equation*}
  A =
  \E \sup_{\Sphere_2^{p-1} \cap \Ball_0^p(d)} \langle Y, b \rangle
\end{equation*}
and $Y$ is a $p$-dimensional standard Gaussian $Y$.
Thus we can reduce the problem to bounding the supremum of a Gaussian process.

Let $\mathcal{N} \subset \Sphere_2^{p-1} \cap \Ball_0^p(d)$ be a minimal
$\delta$-covering of $\Sphere_2^{p-1} \cap \Ball_0^p(d)$ in the Euclidean
metric with the property that for each $x\in\Sphere_2^{p-1} \cap \Ball_0^p(d)$
there exists $y \in \mathcal{N}$ satisfying $\norm{x - y}_2 \leq \delta$ and
$x - y \in \Ball_0^p(d)$. (We will show later that such a covering exists.)

Let $b^* \in \Sphere_2^{p-1} \cap \Ball_0^p(d)$ satisfy
\begin{equation*}
  \sup_{\Sphere_2^{p-1} \cap \Ball_0^p(d)} \langle Y, b \rangle
  = \langle Y, b^* \rangle
  \,.
\end{equation*}
Then there is $\tilde{b} \in \mathcal{N}$ such that $\norm{b^* - \tilde{b}}_2
\leq \delta$ and $b^* - \tilde{b} \in \Ball_0^p(d)$.  
Since $(b^* - \tilde{b}) / \norm{b^* - \tilde{b}}_2 \in \Sphere_2^{p-1} \cap
\Ball_0^p(d)$,
\begin{align*}
    \langle Y, b^* \rangle
  &=
    \langle Y, b^* - \tilde{b} \rangle
  +
    \langle Y, \tilde{b} \rangle
  \\
  &\leq
    \delta \sup_{u \in \Sphere_2^{p-1} \cap \Ball_0^p(d)} \langle Y, u \rangle
    + \langle Y, \tilde{b} \rangle
    \\
    &\leq
      \delta \langle Y, b^* \rangle
      + \max_{b \in \mathcal{N}} \langle Y, b \rangle
  \,.
\end{align*}
Thus,
\begin{equation*}
    \sup_{b \in \Sphere_2^{p-1} \cap \Ball_0^p(d)} \langle Y, b \rangle
  \leq
    (1-\delta)^{-1} \max_{b \in \mathcal{N}} \langle Y, b \rangle
  \,.
\end{equation*}
Since $\langle Y, b \rangle$ is a standard Gaussian for every $b \in \mathcal{N}$, 
a union bound \cite[Lemma 2.2.2]{vanderVaartAndWellner} implies
\begin{equation*}
    \E \max_{b \in \mathcal{N}} \langle Y, b \rangle
  \leq
    c \sqrt{\log \card{\mathcal{N}}}
\end{equation*}
for an absolute constant $c > 0$.  Thus,
\begin{align*}
    \E \sup_{b \in \Sphere_2^{p-1} \cap \Ball_0^p(d)} \langle Y, b \rangle
  &\leq
    c (1-\delta)^{-1} \sqrt{\log \card{\mathcal{N}}}
\end{align*}

Finally, we will bound $\log\card{\mathcal{N}}$ by constructing a
$\delta$-covering set and then choosing $\delta$. It is well known
that the minimal $\delta$-covering of $\Sphere_2^{d-1}$ in the Euclidean
metric has cardinality at most $(1 + 2/\delta)^d$. Associate with each
subset $I \subseteq \{1,\ldots, p\}$ of size $d$, a minimal
$\delta$-covering of the corresponding isometric copy of $\Sphere_2^{d-1}$.
This set covers every possible subset of size $d$, so for each $x \in
\Sphere_2^{p-1}\cap\Ball_0(d)$ there is $y \in \mathcal{N}$ satisfying 
$\norm{x - y}_2 \leq \delta$ and $x - y \in \Ball_0(d)$.
Since there are ($p$ choose $d$) possible subsets,
\begin{align*}
    \log \card{\mathcal{N}}
  &\leq
    \log \binom{p}{d} + d \log(1+2/\delta)
  \\
  &\leq
    \log \left(\frac{p e}{d}\right)^d + d \log(1+2/\delta)
  \\
  &=
    d + d \log(p/d) + d \log(1+2/\delta)
  \,.
\end{align*}
In the second line, we used the binomial coefficient bound $\binom{p}{d} \leq (e p / d)^d$. 
If we take $\delta = 1/4$, then
\begin{align*}
    \log \card{\mathcal{N}}
  &\leq
    d + d \log(p/d)  + d \log 9
  \\
  &\leq
    c d \log(p / d)
  \,,
\end{align*}
where we used the assumption that $d < p / 2$.  Thus,
\begin{equation*}
  A 
  = \E \sup_{\Sphere_2^{p-1} \cap \Ball_0^p(d)} \langle Y, b \rangle
  \leq c d \log(p / d)
\end{equation*}
for all $d \in [1, p/2)$.
\qed

\end{document}